\documentclass[letterpaper]{article} % DO NOT CHANGE THIS
\usepackage{aaai2026}  % DO NOT CHANGE THIS
\usepackage{times}  % DO NOT CHANGE THIS
\usepackage{helvet}  % DO NOT CHANGE THIS
\usepackage{courier}  % DO NOT CHANGE THIS
\usepackage[hyphens]{url}  % DO NOT CHANGE THIS
\usepackage{graphicx} % DO NOT CHANGE THIS
\usepackage{natbib}  % DO NOT CHANGE THIS AND DO NOT ADD ANY OPTIONS TO IT
\usepackage{caption} % DO NOT CHANGE THIS AND DO NOT ADD ANY OPTIONS TO IT
\usepackage{algorithm}
\usepackage{algorithmic}

\usepackage{newfloat}
\usepackage{listings}
\DeclareCaptionStyle{ruled}{labelfont=normalfont,labelsep=colon,strut=off} % DO NOT CHANGE THIS
\floatstyle{ruled}
\newfloat{listing}{tb}{lst}{}
\floatname{listing}{Listing}
\usepackage{amsmath}
\usepackage{amsthm}
\usepackage{amssymb}
\usepackage{makecell}
\usepackage{xcolor}
\usepackage{booktabs}
\usepackage{array}
\usepackage{multirow}
\usepackage{dsfont}
\usepackage{tocloft}

\newtheorem{definition}{Definition}
\newtheorem{proposition}{Proposition}

\title{Copyright Infringement Detection in Text-to-Image Diffusion Models via Differential Privacy}
\author{
    Xiafeng Man\textsuperscript{\rm 1},
    Zhipeng Wei\textsuperscript{\rm 3,\rm 4},
    Jingjing Chen\textsuperscript{\rm 2}\thanks{Corresponding author.}
}
\affiliations{
    \textsuperscript{\rm 1}College of Future Information Technology, Fudan University, Shanghai, China\\
    \textsuperscript{\rm 2}Institute of Trustworthy Embodied AI, Fudan University, Shanghai, China\\
    \textsuperscript{\rm 3}International Computer Science Institute, CA, USA\\
    \textsuperscript{\rm 4}UC Berkeley, CA, USA\\
    xfmanacad@outlook.com, zwei@icsi.berkeley.edu, chenjingjing@fudan.edu.cn
}

\begin{document}

\maketitle

\begin{abstract}
	The widespread deployment of large vision models such as Stable Diffusion raises significant legal and ethical concerns, as these models can memorize and reproduce copyrighted content without authorization. Existing detection approaches often lack robustness and fail to provide rigorous theoretical underpinnings. To address these gaps, we formalize the concept of copyright infringement and its detection from the perspective of Differential Privacy (DP), and introduce the conditional sensitivity metric, a concept analogous to sensitivity in DP, that quantifies the deviation in a diffusion model's output caused by the inclusion or exclusion of a specific training data point. To operationalize this metric, we propose \textbf{D-Plus-Minus (DPM)}, a novel post-hoc detection framework that identifies copyright infringement in text-to-image diffusion models. Specifically, DPM simulates inclusion and exclusion processes by fine-tuning models in two opposing directions: learning or unlearning.	Besides, to disentangle concept-specific influence from the global parameter shifts induced by fine-tuning, DPM computes confidence scores over orthogonal prompt distributions using statistical metrics. Moreover, to facilitate standardized benchmarking, we also construct the \textbf{C}opyright \textbf{I}nfringement \textbf{D}etection \textbf{D}ataset (CIDD), a comprehensive resource for evaluating detection across diverse categories. Our results demonstrate that DPM reliably detects infringement content without requiring access to the original training dataset or text prompts, offering an interpretable and practical solution for safeguarding intellectual property in the era of generative AI.
\end{abstract}

\begin{links}
    \link{Project Page}{https://leo-xfm.github.io/pubs/dpm}
\end{links}

\section{Introduction}
	\begin{figure*}[!t]
		\centering
		\includegraphics[width=0.98\textwidth]{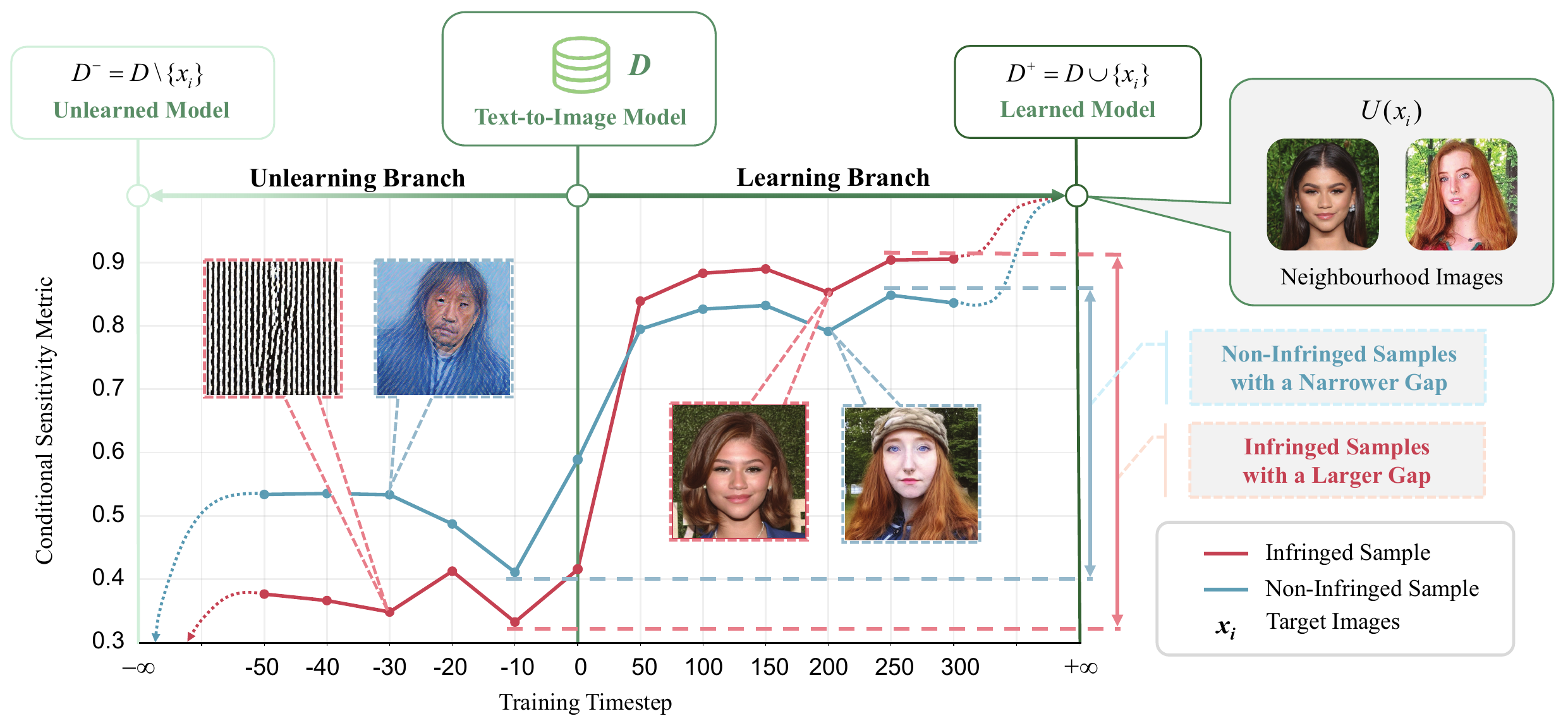} 
		\captionof{figure}{
			D-Plus-Minus Method. Given the neighbourhood images $U(x_i)$, i.e., several images of similar semantics extracted from the target image, of the target image $x_i$ as the training subset, we fine-tune the text-to-image model $G$ towards two branch: learning branch $G_{D^+}$ and unlearning branch $G_{D^-}$. Experimental results show that infringed samples lead to a significant shift in sensitivity metric, whereas non-infringed samples only cause minor changes.
		}
		\label{fig: DPM method}
	\end{figure*}

Recent advances in large vision models have improved the realism of image synthesis. While they are widely adopted in creative industries and public platforms, they have also raised serious concerns over copyright infringement. 
Researchers~\cite{somepalli2023diffusion, cilloni2023privacythreatsstablediffusion, carlini2023extracting} find that models such as Stable Diffusion~\cite{rombach2022high} may memorize and reproduce contents in training datasets, including those with unclear permission or copyright violation.
These risks are exacerbated by the lack of transparency regarding the provenance of training data in most models. As a result, there is an urgent need for accurate and reliable post-hoc methods to detect potential copyright infringements in models.

Recent studies have attempted to solve this issue. For example, CopyScope~\cite{zhou2023copyscopemodellevelcopyrightinfringement} proposes a model-level framework to quantify each component's contribution to potential copyright infringement in diffusion workflows, by evaluating FID-based similarity and Shapley value of attributing responsibility. However, it does not identify specific infringed concepts or samples, and therefore cannot provide concrete legal evidence of infringement.
Other works~\cite{wang2024evaluatingmitigatingipinfringement, xu2025largevisionlanguagemodelsdetect} detect infringement via prompt queries or prompt engineering work. 
However, prompt-based detection is inherently fragile, as it depends on constructing specific conditioning prompts to trigger infringing outputs, which can be easily affected by minor changes, including model updates and sampling randomness. 
This limits generalization across models and datasets. It also lacks interpretability, as the underlying causes of infringement remain unclear, and reproducibility, due to stochastic generation and non-deterministic model behavior.

To this end, we propose a novel perspective on copyright infringement, grounded in the theory of differential privacy. 
We reinterpret the detection of copyright infringement as the compliance with or violation of conditional differential publicity. Specifically, when a particular concept, such as the neighborhood images of a target image, is present or absent in the training data, it can significantly alter the model’s output in response to prompts associated with that concept.
This leads to the definition of a new metric, \textbf{conditional sensitivity}, that quantifies the extent of publicity. It allows us to formalize the infringement criteria based on measurable behavioral changes. 

Building on the theoretical foundation, we further propose a detection framework, \textbf{D-Plus-Minus (DPM)}, which identifies potential copyright infringement by evaluating a conditional sensitivity metric with respect to a specific concept.
Specifically, to estimate the model’s dependency on the target concept, DPM simulates its inclusion and exclusion through two parallel fine-tuning branches: a \emph{learning branch}, where the model is encouraged to memorize the concept, and an \emph{unlearning branch}, where the model is trained to forget it. 
The values of CLIP-based embeddings are then compared with those of the original model, deriving an empirical conditional sensitivity. 
Moreover, the fine-tuning process inevitably affects the output behavior of a diffusion model, even for content unrelated to the target concept, which can compromise the reliability of our sensitivity metric. To address this, we statistically align the empirical sensitivity to the ideal one by referring to the orthogonal sensitivity.
We visualize the discrepancy in conditional sensitivity in Fig.~\ref{fig: DPM method}, where the larger change observed in infringed samples compared to non-infringed ones validates its use as a reliable measurement.

To evaluate the detection framework, we construct the \textbf{Copyright Infringement Detection Dataset (CIDD)} for different Large Vision Models (LVMs) and Large Vision and Language Models (LVLMs). It covers three high-risk categories: human face, architecture, and arts painting, with potentially infringing and non-infringing content. 
Our experiments against four models show that DPMs all yield weighted average AUC values above 80\%. 

In summary, our work makes the following contributions:

\begin{itemize}
	\item To our knowledge, we are the first to introduce differential privacy, a theoretical guarantee, into the notion of copyright infringement and its detection.
	\item We propose a theoretically and statistically grounded DPM framework for the post-hoc detection of copyright infringement. DPM simulates the inclusion and exclusion in two opposing fine-tuning branches: learning and unlearning, to measure the conditional sensitivity.
	\item With the construction of the CIDD dataset, DPM consistently shows excellent performance in terms of AUC and interpretability. Moreover, CIDD offers controllable samples, diverse classes, and well-aligned clean pairs, making it a valuable resource for advancing research on copyright detection.
\end{itemize}

\section{Background and Related Work} \label{sec: related work}

As artificial intelligence-generated content (AIGC) becomes popular in daily life and creative workflows, legal and ethical concerns over copyright infringement are now emerging as an urgent issue. Recent cases illustrate the conflicts between AIGC and intellectual property (IP) rights. For instance, Studio Ghibli sued OpenAI for IP infringement over ChatGPT-generated images that replicated its artistic style; Disney and NBC Universal both accused Midjourney of IP infringement for generating images resembling copyrighted works such as \textit{The Simpsons}, \textit{The Avengers}, and \textit{Toy Story}.

To address these challenges, several regulatory frameworks~\cite{airmf2023, interimgai2023, euaiact2024, generalpurposeai2025} are being introduced globally. However, they are difficult to enforce in practice due to the opacity of large-scale models. Once a model has been trained, it is impractical to retrain from scratch to satisfy these regulations. Hence, it becomes vital to determine whether a specific data point has been memorized and is urgently needed for post-hoc methods to detect and remove unauthorized content from deployed models.

On the technical front, \citet{zhou2023copyscopemodellevelcopyrightinfringement} propose CopyScope—a model-level framework that quantifies copyright infringement in the full diffusion workflow in three stages: identify influential components, quantify by FID-Shapley and evaluate contributions of models. 
DIAGNOSIS~\cite{wang2023diagnosis} detects the marked copyright infringement by first coating the protected dataset, then approximating the memorization strength, and making a hypothesis testing.
\citet{ma2024datasetbenchmarkcopyrightinfringement} provides a dataset and benchmark for copyright infringement protection in data unlearning of text-to-image diffusion models. It introduces a metric that evaluates similarity between two images from both semantic and stylistic perspectives. 

In the copyright infringement detection of LVLMs, \citet{wang2024evaluatingmitigatingipinfringement} propose a prompt engineering method that generates prompts to trigger IP violations in Large Language Models (LLMs) under black-box settings.
\citet{xu2025largevisionlanguagemodelsdetect} introduce an IP benchmark dataset and find out that LVLMs tend to misclassify the negative IP samples using in-context learning.
\citet{Chiba_Okabe_2025} design a systematic evaluation criterion to quantify and estimate the originality level of data, and introduce PREGen to modify outputs and lower originality values.

Despite the promising progress, these approaches face notable limitations. Model-based detection cannot localize infringement to the specific concepts, which is difficult to establish a clear causal link between the model behavior and particular copyrighted content; detecting via prompt query lacks theoretical interpretability and may be inaccurate due to the hallucinations and defense algorithms of LLMs.
In conclusion, it is necessary to demonstrate precise evidence of memorization and model behavior in the region of high-stakes legal or forensic settings.

\section{Preliminaries}

\subsection{Diffusion Models}

Diffusion Models (DMs)~\citep{ho2020denoising} are latent variable generative models which learns a data distribution $p(x)$ by gradually denoising variables sampled from a Gaussian distribution. This process is achieved by reversing a fixed Markov chain of stochastic noise injection. Formally, the forward process gradually adds noise to a data sample  $x_0 \sim q(x_0)$ over $T$ steps through a sequence of Gaussian transitions:
\begin{equation}
	q(x_t | x_{t-1}) = \mathcal{N}(x_t; \sqrt{1 - \beta_t} x_{t-1}, \beta_t I),
\end{equation}
where $\beta_t$ is a variance schedule, and $t$ is sampled from $\{1,\dots,T\}$. The generative model $\epsilon_\theta$ is then trained to reverse this process by predicting the noise $\epsilon$ from a noisy sample $x_t$. The training objective can be simplified to:
\begin{equation}
	\mathcal{L}_\text{{DM}} = \mathbb{E}_{x, \epsilon \sim \mathcal{N}(0,1), t} \left[ \left\| \epsilon - \epsilon_\theta(x_t, t) \right\|_2^2 \right].
\end{equation}

To reduce model complexity and preserve image details, Latent Diffusion Models (LDMs)~\cite{rombach2022high} employ the DM training in the lower-dimensional latent space, mapped by an encoder and reconstructed by a decoder.

\subsection{Unlearning of Diffusion Models}
Machine unlearning~\cite{bourtoule2021machine} is a task of removing the influence of specific dataset $A$ from a model trained on $X$ without retraining from scratch. 
In DMs, Naive Deletion fine-tunes the model on the retained dataset $X \setminus A$ by minimizing the simplified evidence-based lower bound (ELBO); other approaches include NegGrad~\cite{golatkar2020eternal} and SISS~\cite{alberti2025dataunlearningdiffusionmodels}.

\subsection{Differential Privacy}

Differential privacy (DP)~\cite{dwork2006calibrating} is a formal notion of algorithmic privacy, which aims to prevent the release of private information.

Let $D_0 \in \mathcal{W}^m$ denote a database from the input domain $\mathcal{W}$, and two databases $D, D' \in \mathcal{W}^m$ are said to be neighbouring datasets if $d(D, D') \leq 1$, where $d$ represents the distance between two datasets. Differential privacy is formally defined as follows:

\begin{definition}[Differential Privacy] \label{def: DP}
	A randomized algorithm $M: \mathcal{W}^m \rightarrow S$ is said to satisfy \emph{$(\epsilon, \delta)$-differential privacy} if for every pair of neighbouring databases $D$ and $D'$, and for all possible sets of outputs $S$:
	\begin{equation} \label{eq: DP}
		\Pr[M(D) \in S] \leq e^{\epsilon} \cdot \Pr[M(D') \in S] + \delta, 
	\end{equation}
	with probability $1-\delta$, where $\epsilon$ is the privacy budget and $\delta$ is a failure probability for the definition (typically $\delta \leq \frac{1}{n^2}$), which means the privacy guarantee cannot hold with probability $\delta$.
\end{definition}

For a query $M: D_0 \rightarrow \mathcal{R}$, the global sensitivity (GS) is the maximum distance for any two neighbouring datasets:
\begin{equation} \label{eq: GS}
	GS(M) = \max_{D, D': d(D,D') \leq 1} |M(D) - M(D')|, 
\end{equation}
and the local sensitivity of $M$ at $D: D_0$ fixes $D$ to be the actual dataset being queried, and considers all of its neighours:
\begin{equation} \label{eq: LS}
	LS(M, D) = \max_{D': d(D,D') \leq 1} | M(D) - M(D') |. 
\end{equation}

\section{Copyright Infringement} \label{sec: overview copyright infringement}

Copyright infringement occurs when an algorithm $G$ (e.g., a text-to-image model) is trained on a dataset $D$ that includes a subset of copyrighted or unauthorized data samples $D_C = \{ x_{c_i} \} \subset D$, without permission. 
A claim of infringement is typically established by showing that model's outputs reproduce or are substantially similar to the protectable expressive elements of the copyrighted samples in $D_c$.

\subsection{Formalization of Differential Privacy}  \label{sec: Formalization of Differential Privacy}

Some researchers~\cite{cilloni2023privacythreatsstablediffusion, somepalli2023diffusion, carlini2023extracting} have revealed significant privacy vulnerabilities in the outputs of diffusion models. For instance, membership inference~\cite{cilloni2023privacythreatsstablediffusion} has achieved success rates exceeding 60\%; data extraction and reproduction~\cite{somepalli2023diffusion, carlini2023extracting} have shown that diffusion models can memorize and reproduce individual training images. It suggests that these models exhibit almost \emph{no} conditional differential privacy regarding the training dataset, but with more publicity:

\begin{definition}[Conditional Publicity for Diffusion Models]
	A diffusion model $G(p)$ conditioned on an input $p$ (e.g., a text prompt) is said to satisfy \emph{$\epsilon$-conditional publicity} if there exist neighbouring training datasets $D$ and $D'$ that differ in a single element (which is corresponding to a concept within the semantic neighborhood $U(p)$ of the input $p$), and there exists at least one measurable subset $S \subseteq \{G(p_i) \mid p_i \in U(p)\}$ such that:
	\begin{equation} \label{eq: ConditionalPrivacy-T2I}
		\begin{aligned}
			\Pr[G(\theta_D, p) \in S] &> e^\epsilon \cdot \Pr[G(\theta_{D'}, p) \in S] \\
			&\gg \Pr[G(\theta_{D'}, p) \in S],
		\end{aligned}		
	\end{equation}
	where $\theta_D$ and $\theta_{D'}$ are the model parameters obtained by training model $G$ on datasets $D$ and $D'$ respectively. 
\end{definition}

Here, $\epsilon > 200$ indicates a substantial violation of privacy, i.e., the model output strongly depends on the presence of a specific concept in the training set.

From the perspective of differential privacy, a generative model's relevant outputs could be drastically altered when a single copyrighted training concept is modified. Conversely, an input that the model has not seen, namely, non-infringed data, could be relatively private. 

A detailed mathematical description of differential privacy in diffusion models is given in the appendix.

\subsection{Definition of Copyright Infringement}

Building upon the formalization of differential privacy, we now define the mathematical criteria for: when a generative model $G$ can be considered to have infringed upon a specific data point or concept (e.g., a copyrighted image or IP), and when it does not, based on the model behavior.

\begin{definition}[Copyright Infringement] \label{def: infringement}
	Let $x_c \in D_C$ denote a copyrighted data point or concept, and $p$ be an input (e.g., a text prompt) semantically aligned with $x_c$. We say that model $G$ trained on $D$ \emph{infringes} upon $x_c$ if there exists a measurable subset $S \subseteq \{G(p_i) \mid p_i \in U(p)\}$ such that:
	\begin{equation}
		\Pr[G(\theta_D, p) \in S] \gg \Pr[G(\theta_{D'}, p) \in S],
	\end{equation}
	where $D' = D \setminus \{x_c\}$ is a neighboring dataset.
\end{definition}

This definition implies that the relevant output of $G$ is significantly influenced by the presence of $x_c$, thereby violating copyright.

\begin{definition}[Copyright Non-Infringement] \label{def: noninfringement}
	Let $x$ be a non-infringed data point or concept such that $x \notin D$ for all training datasets considered. We say that model $G$ \emph{does not infringe} upon $x$ if for any input $p$ and for all measurable subsets $S \subseteq \{G(p_i) \mid p_i \in U(p)\}$ such that:
	\begin{equation}
		\Pr[G(\theta_D, p) \in S] = \Pr[G(\theta_{D'}, p) \in S].
	\end{equation}
	where $D$ and $D'$ are any neighboring training datasets, and $\theta_D$, $\theta_{D'}$ denote the model parameters trained on $D$ and $D'$ respectively.
\end{definition}

This definition ensures that the output distribution is invariant with respect to the inclusion or exclusion of $x$, thereby guaranteeing that $x$ does not influence model behavior.
Taken together, these definitions provide a theoretical framework for copyright infringement detection.

\section{Detection of Copyright Infringement}

\subsection{Definition of Infringement Detection}

Within the theoretical framework described above, infringement can be assessed by measuring distributional changes in the model's output caused by the presence of a specific training data or concept. 
A significant change implies that memorization has occurred, suggesting potential infringement, whereas invariance suggests no infringement.
\begin{definition}[Detection of Copyright Infringement]
	\emph{Detection of copyright infringement} is the estimation of the likelihood that a specific visual input $\hat{x}_i \in \mathcal{X}$ is used in, or is memorized by the model during training on dataset $D$, and its influence is manifested in the model's output behavior. Formally, we define the detection task as a confidence scoring function range in $[0,1]$:
	\begin{equation} \label{eq: CI}
		f(\hat{x}_i) = \mathbb{P} \left\{ \hat{x}_i \in D \ \text{or} \ \tau(\hat{x}_i, G) \right\},
	\end{equation}
	where $\hat{x}_i$ is the query image or visual concept to be evaluated; $D$ is the training dataset (possibly unknown); $\tau(\hat{x}_i, G)$ is an influence function that quantifies the dependence of the model $G$ on $\hat{x}_i$, based on metrics such as semantic similarity, gradient attribution, or reconstruction closeness.
\end{definition}
A high confidence score indicates potential copyright infringement, while a low one suggests low likelihood.

\subsection{Conditional Sensitivity Metric}

To quantify function $f(\hat{x}_i)$ in Eq.~\ref{eq: CI}, we introduce the notion of \emph{conditional sensitivity} as a principal metric for standardizing the confidence score of copyright infringement, analogous to local sensitivity in Eq.~\ref{eq: LS} in differential privacy. Sensitivity captures how much a query function $M$ depends on a specific training sample:
\begin{equation} \label{eq: CS}
	CS(M, \hat{x}_i) = \max_{D, D': D \triangle D' \leq \{\hat{x}_i\}} \left| M(D) - M(D') \right|,
\end{equation}
where $D$ and $D'$ are neighboring datasets that differ by the inclusion or exclusion of the conditional datapoint $\hat{x}_i$, and the function $M(D)$ denotes the output of a query function when trained on dataset $D$.

\subsection{D-Plus-Minus (DPM) Detection}

\begin{figure*}[t]
	\centering
	\includegraphics[width=0.98\textwidth]{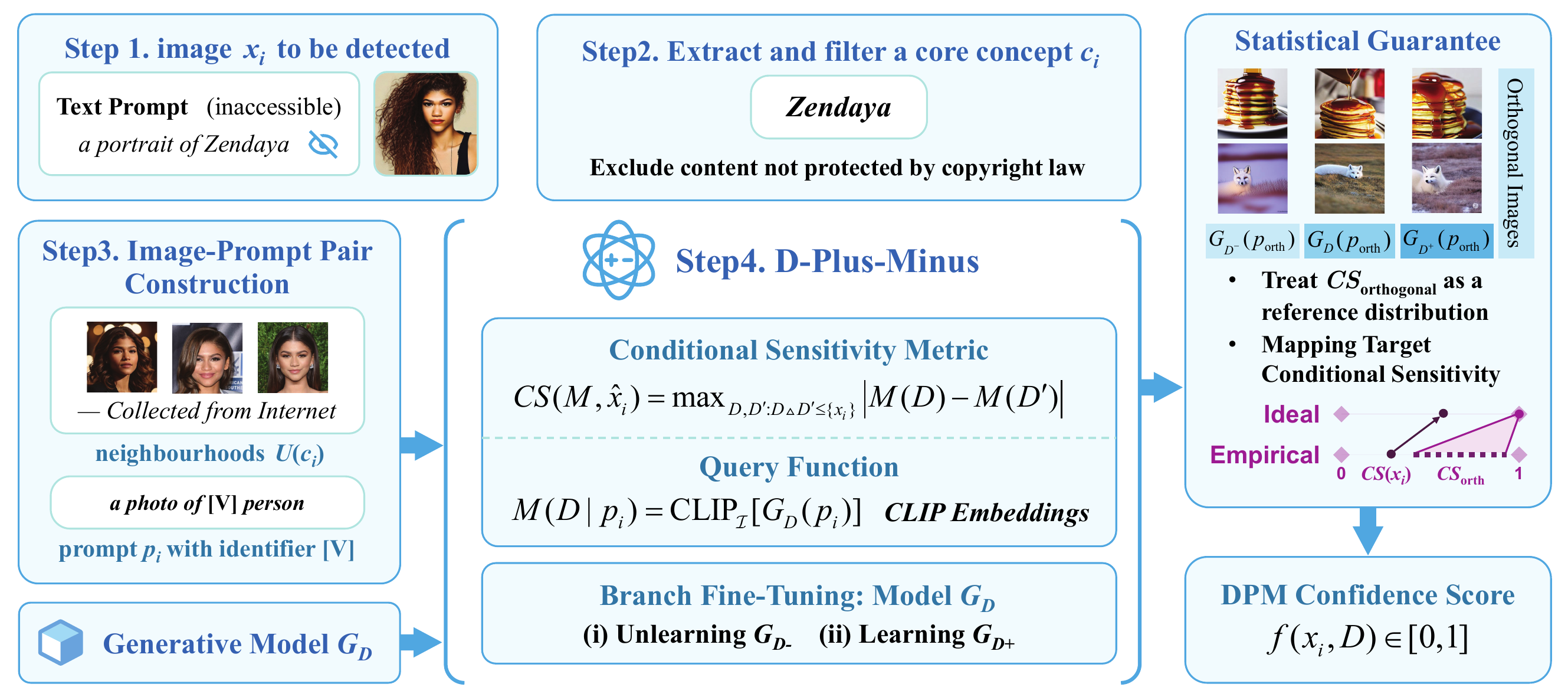} 
	\caption{Detection Procedure of Copyright Infringement. Firstly, we extract a concept from the target image. Next, we collect several images associated with this concept to form a neighborhood subset, and construct a prompt using a unique identifier (e.g., a photo of [V] person). Finally, we feed these image-prompt pairs into the D-Plus-Minus platform to compute a DPM score with statistical guarantee.}	
	\label{fig: procedure}
\end{figure*}

\paragraph{Problem Setting} \label{sec: setting}
Detection of copyright infringement faces several practical challenges, such as scalability, inaccessibility of training data, conditional input unavailability, and insufficient theoretical guarantees.
In light of these issues, our work operates under a realistic and challenging set of assumptions:

(1) white-box access to a pretrained model; 

(2) absence of corresponding input prompt; 

(3) inaccessibility of training data. 

Given these constraints,  we aim to identify evidence of infringement as reflected in the model’s observable behavior, grounded in the above theoretical framework.
Figure~\ref{fig: procedure} illustrates the core procedure of our proposed D-Plus-Minus detection framework. It contains several steps: concept extraction, image collection, branch training, conditional sensitivity measurement, statistics analysis, and branch merging.

\paragraph{Pre-Processing: Concept extraction and Image Collection}
Given a target image $\hat{x}_i$, we firstly extract and filter a core concept to exclude non-detected contents. As fine-tuning model needs several images related to the concept, we construct a neighborhood $U(\hat{x}_i)$ of the target concept as our training dataset, consisting of similar semantics to be detected, and then specify a general prompt $p_i$ (format as: a photo of [V] [class]) with identifier (e.g., “[V]”, “sks”). 

\paragraph{Branch Training}
Since the presence or absence of the target data point $\hat{x}$ in the training dataset is unknown, we simulate its inclusion and exclusion through model fine-tuning: one is the learning branch where the model $G$ is fine-tuned to include $\hat{x}_i$ (denoted $G_{D^+}$), and the other one is the unlearning branch where the model aims to remove it (denoted $G_{D^-}$). The branch objective can be defined as:
{\small
	\begin{equation} \label{eq: d+-, simplified}
		\mathcal{L}_{\text{branch}}(x_i) 
		= I \cdot 
		\mathbb{E}_{x_i, p, \epsilon, t} 
		\left[ w_t \left\| G(\alpha_t x_i + \sigma_t \epsilon, p_i) 
		- x_i \right\|_2^2 \right],
	\end{equation}
}
where $\alpha_t$, $\sigma_t$, $w_t$ are functions of the diffusion timestep $t \sim U([0,1])$, controlling the noise schedule and denoising weight, and $I$ denotes the branch indicator:
\begin{equation}
	I =
	\begin{cases}
		+1, & \text{for the learning branch} \\
		-1, & \text{for the unlearning branch}
	\end{cases} .
\end{equation}

\paragraph{Conditional Sensitivity Measurement}
To assess the effect of $\hat{x}_i$ on the model's generation behavior, we compare the outputs between a fine-tuned model $G_{D^*}$, i.e., $G_{D^+}$ or $G_{D^-}$, and $G_D$ under the same text prompt $p_i$, and specifically define the conditional sensitivity metric via cosine similarity:
{\small
	\begin{equation}
		\begin{aligned}
			CS(M, \hat{x}_i, D^*) = 
			\max_{\substack{D, D' :
					\\ D \triangle D' \leq \{\hat{x}_i\}}}
			\frac
			{M(G_D \mid p_i) \cdot M(G_{D^*} \mid p_i)}
			{\left\| M(G_D \mid p_i) \right\| \left\| M(G_{D^*} \mid p_i) \right\|}.
		\end{aligned}
	\end{equation}
}
Here, \( M(\cdot) = \text{CLIP}_\mathcal{I}(\cdot) \) denotes the CLIP~\cite{radford2021learning} image encoder, which serves as a query function to capture the semantics similarity between two outputs. The nearer the conditional sensitivity is to $1$, the less sensitive it is. 

Considering that reaching absolute learning or unlearning such that $D \triangle D' = \{\hat{x}_i\}$ is impossible and impractical, we select fine-tuned models in several training timesteps for measurement, which means $D \triangle D' < \{\hat{x}_i\}$. Here, “$<$” implies that the model is not fully trained or memorized ${\hat{x}_i}$ on $D'$ compared to $D$, and it still satisfies the requirement in Eq.~\ref{eq: CS} that $D \triangle D' \leq \{\hat{x}_i\}$.

\paragraph{Statistics Analysis}
As fine-tuning in both branches will inevitably alter model's generalization behavior, especially on unrelated content, we construct a reference distribution by generating orthogonal images, clarifying the global parameter shifts. 
To statistically standardize the conditional sensitivity score across different target samples, we normalize it by the average conditional sensitivity over the orthogonal set. Ideally, the conditional sensitivity among orthogonal samples should be $1$ (exactly the same outputs), leading to:
\[
{CS(M,{x}_i, D^*)}:{\overline{CS(M, {X}_\text{orth}, D^*)}}
= {\hat{CS}(M,\hat{x}_i, D^*)}:1 
\]
\begin{equation}
	\Leftrightarrow \hat{CS}(M,\hat{x}_i, D^*) = 
	\frac{CS(M,{x}_i, D^*)}
	{\overline{CS(M, {X}_\text{orth}, D^*)}},
\end{equation}
where ${X}_\text{orth}$ denotes the orthogonal sample set, and $\hat{CS}$ denotes the ideal conditional sensitivity. It aligns the difference among fine-tuning models of different samples, and provides a statistically grounded reference rather than an accuracy-based one.

\begin{table*}[!t]	
	\setlength{\tabcolsep}{8pt}
	\centering
	
	\begin{tabular}{c cc cc cc cc}
		\Xhline{1.2pt}
		\multirow{2}{*}{\textbf{Class}} 
		& \multicolumn{2}{c}{\textbf{SD1.4}} 
		& \multicolumn{2}{c}{\textbf{SDXL-1.0}}
		& \multicolumn{2}{c}{\textbf{SANA-0.6B}}
		& \multicolumn{2}{c}{\textbf{FLUX.1}} \\
		& AUC $\uparrow$ & SoftAcc $\uparrow$ 
		& AUC $\uparrow$ & SoftAcc $\uparrow$ 
		& AUC $\uparrow$ & SoftAcc $\uparrow$
		& AUC $\uparrow$ & SoftAcc $\uparrow$ \\
		
		\midrule 
		Human Face    & 0.9011 & 0.8058 & 0.7011 & 0.6289 & 0.8062 & 0.7285 & 0.7531 & 0.6419\\
		Architecture  & 0.8021 & 0.7106 & 0.9256 & 0.8488 & 0.9043 & 0.8224 & 0.9500 & 0.8606\\
		Arts Painting & 0.8555 & 0.7604 & 0.8881 & 0.8550 & 0.8140 & 0.7204 & 0.7326 & 0.6935\\
		\midrule
		Weighted Average & 0.8584 & 0.7644 & 0.8170 & 0.7523 & 0.8398 & 0.7571 & 0.8122 & 0.7247\\
		Merged Total  & 0.8071 & 0.6726 & 0.7800 & 0.7234 & 0.7914 & 0.6855 & 0.8257 & 0.7039\\
		\Xhline{1.2pt}	\end{tabular}
	
	\caption{Quantitative Detection Metrics. Models are run separately on the classes of CIDD dataset in different models. \textit{Merged Total} means that the $\Delta CS(\cdot)$ are normalized altogether, while others are normalized within the class.}
	
	\label{table: quantitative}
\end{table*} 

\begin{table}[!t]
	\setlength{\tabcolsep}{12pt} 
	\centering
	
	\begin{tabular}{ccc}
		\Xhline{1.2pt}
		\textbf{Level} & \textbf{Categories} & \textbf{Examples} \\
		\midrule
		Level 1        & Technics   & — \\
		Level 2        & Content    & Human Face* \\
		Level 3-1      & Structure  & Architecture* \\
		Level 3-2      & Style      & Arts Painting* \\
		Level 4        & Semantics  & Plots \& Themes \\
		\Xhline{1.2pt}
	\end{tabular}
	
	\caption{Hierarchical Categories of Copyright Infringement. It is ordered from low-level perceptual features to high-level conceptual constructs. “*” means the classes in CIDD.}
	\label{table: Hierarchical Categories}
\end{table}

\begin{figure}[!t]
	\centering
	\includegraphics[width=0.48\textwidth]{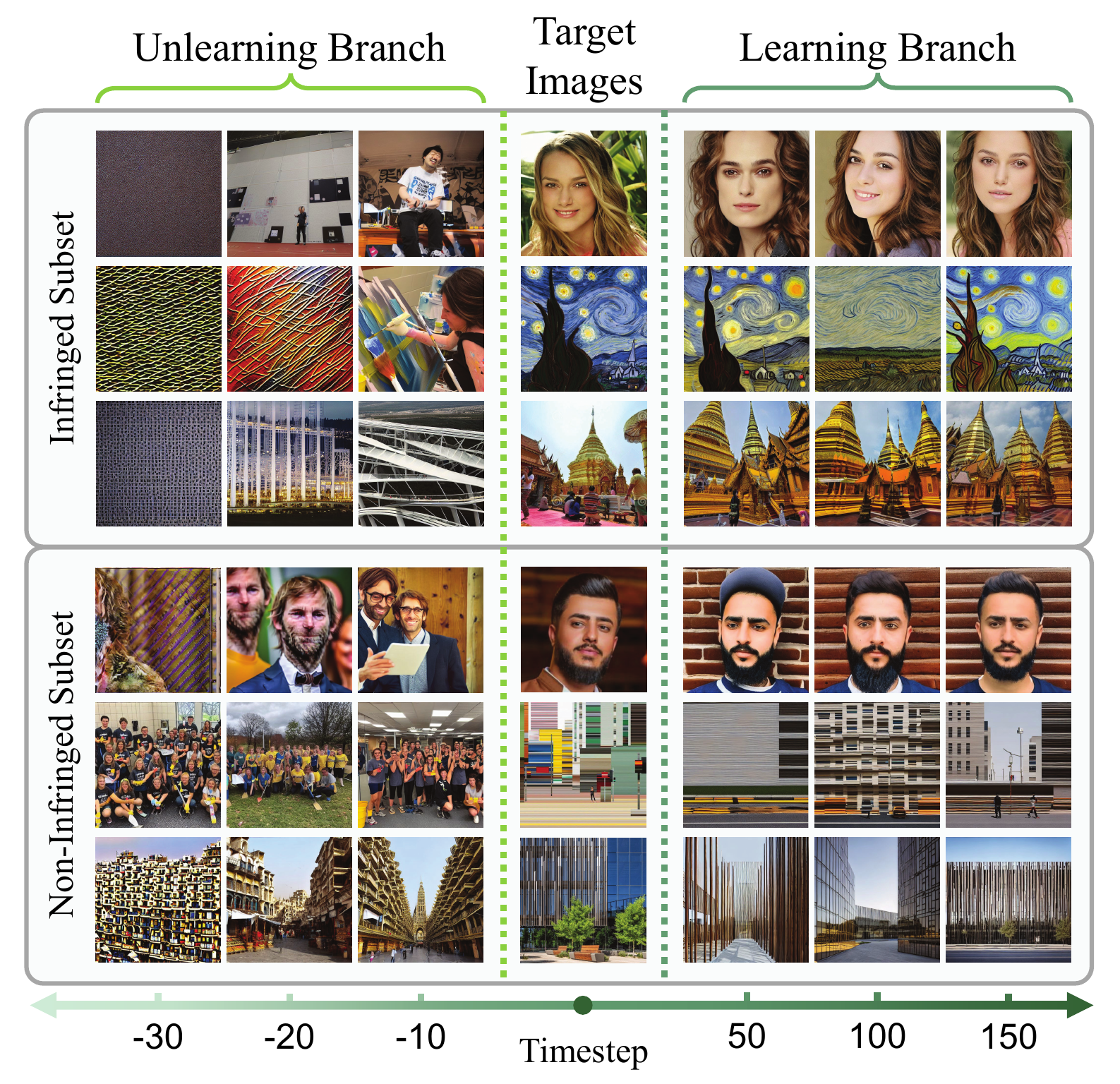} 
	\caption{Qualitative visualization of two branches across different timesteps. Models tend to learn and unlearn faster with infringed samples, while slower on non-infringed ones, and cannot learn exact elements in the target images.}	
	\label{fig: branch visual}
\end{figure}

\paragraph{Branch Merging}
By merging the two branches together, the D-Plus-Minus score can be finally written as:
\begin{align}
	\quad &\text{DPM} (M, \hat{x}_i, D_{\text{total}}) \nonumber\\ = &\sigma 
	\left[ \alpha \cdot
	\frac
	{\Delta \hat{CS}(M, \hat{x}_i) - 
		\min(\Delta \hat{CS}_{\text{class}})}
	{\max(\Delta \hat{CS}_{\text{class}}) - 
		\min(\Delta \hat{CS}_{\text{class}})} 
	\right],
\end{align}
where $\Delta \hat{CS}(M, \hat{x}_i) = \hat{CS}(M, \hat{x}_i, D^+) - \hat{CS}(M, \hat{x}_i, D^-)$ denotes the contrastive sensitivity between two branches, $\sigma(\cdot)$ is the Sigmoid function, and $\alpha>0$ is a scaling coefficient that controls the sharpness of the score mapping. DPM score reflects the model's total conditional sensitivity to the presence and absence of $\hat{x}_i$. Higher scores will suggest potential memorization and infringement.

\section{Hierarchical Categories of Infringement}

To comprehensively categorize copyright infringement in generative models, we propose a hierarchical taxonomy containing four levels of content resemblance, as shown in table~\ref{table: Hierarchical Categories}. 
While low-level technical features, such as edge detectors, texture filters or color histograms, may cause perceptual similarity, they are typically insufficient on their own to constitute legal infringement. Instead, it is more likely to be substantiated when higher-level similarities, such as character compositions or stylistic elements, indicate derivation from copyrighted works. Furthermore, the highest level of semantics similarity is often shown in time-sequenced media (e.g., videos), rather than being confined to one image.

\section{Copyright Infringement Detection Dataset}

To overcome the limitations of existing datasets, we construct the \textbf{C}opyright \textbf{I}nfringement \textbf{D}etection \textbf{D}ataset (CIDD). 
It contains several classes of orthogonal prompts and three image classes that are most likely to be infringed: human face, architecture, and arts painting, with a total of 429 concepts and 2,397 images. Each class is mapped to one hierarchical category in table~\ref{table: Hierarchical Categories}. 
Crucially, CIDD includes both infringed and non-infringed concepts, each of which is annotated with a binary infringement label based on its source and content provenance, and is paired with 3 to 6 neighbourhood images,
enabling robust learning and evaluation under weak and probabilistic assumptions. 

Data examples, summaries, and collection methods are detailed in the appendix.

\begin{figure*}[h]
	\centering
	\includegraphics[width=1\textwidth]{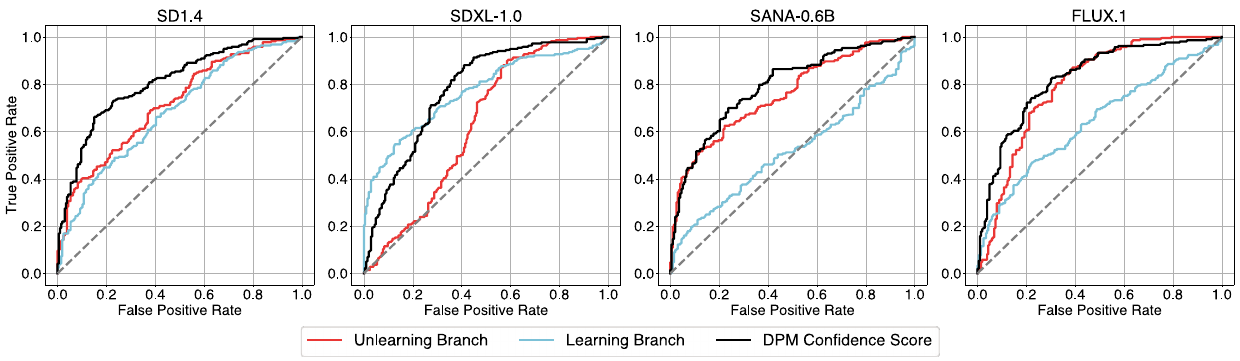}
	% Reduce the figure size so that it is slightly narrower than the column.
	% trim=left bottom right top
	\caption{ROC curves in four representative models. The proposed DPM confidence score consistently outperforms individual branches in terms of AUC, demonstrating its superior capability in distinguishing infringed from non-infringed samples.
	}
	\label{fig: roc-all}
\end{figure*}

\begin{figure}[!t]
	\centering
	\includegraphics[width=0.41\textwidth]{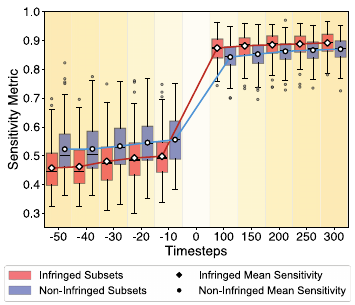} 
	\caption{Conditional Sensitivity in SD1.4. Infringed samples are more sensitive to the model outputs behavior.}
	\label{fig: sens sd1.4}
\end{figure}

\section{Experiment}
As mentioned in Section~\ref{sec: related work}, Copyscope cannot distinguish image-level infringement; DIAGNOSIS is an infringement detection method based on a priori dataset coating; CPDM can only evaluate the extent of infringement after data unlearning. 
These methods target unrealistic settings and different goals, making direct comparison with DPM inappropriate.
As for detection in LLMs and LVLMs, DPM can also be extended to this setting, which we leave as a promising work in the future.
In this section, we evaluate the proposed DPM on the CIDD dataset across four text-to-image models: Stable Diffusion v1.4 (SD1.4)~\cite{rombach2022high}, Stable Diffusion XL (SDXL)~\cite{podell2023sdxlimprovinglatentdiffusion}, SANA~\cite{xie2024sana}, and FLUX~\cite{labs2025flux1kontextflowmatching} models.

We report both quantitative and qualitative results, and perform an ablation study to validate all modules in DPM. 
For the quantitative evaluation, we employ two standard metrics: ROC-AUC, which measures the ability to distinguish infringed from non-infringed samples, and Soft Accuracy (SoftAcc), a stricter metric that evaluates the numerical alignment between confidence scores and ground truth labels, penalizing slight deviations from the correct value.

\subsection{Quantitative Experiments}

We report per-class and aggregated results in Table~\ref{table: quantitative}. DPM consistently achieves a weighted-average AUC above 80\% (as shown in Fig.~\ref{fig: roc-all}) and SoftAcc above 72\% across models, demonstrating its strong generalization.

Moreover, detection performance varies across classes, mainly driven by the distinct generalization capabilities of the CLIP models, the specific parameter configurations of the diffusion models, and the utilized fine-tuning timesteps.

\subsection{Qualitative Experiments}

To better understand the behaviors of our two-branch framework, we visualize representative samples in Fig.~\ref{fig: branch visual}. 

These results suggest that the infringed samples have triggered the model’s internal memory, causing it to reproduce memorized content with high fidelity; while the unlearning branch rapidly suppresses such resemblance, producing visually dissimilar outputs after just a few negative steps.
For non-infringed samples, both branches tend to maintain stable generations with limited directional change.

\subsection{Ablation Study and Robustness}

We provide a detailed ablation study and robustness analysis in the appendix. We find that: 
(1) the proposed conditional sensitivity metric effectively distinguishes between infringed and non-infringed samples, as partially illustrated in Fig.~\ref{fig: sens sd1.4}; 
(2) merging of two branches and multiple timesteps both improve and stabilize detection performance;
(3) image degradation affects little to the detection performance.

\section{Conclusion}

In this paper, we formalize copyright infringement with differential privacy and introduce DPM, a principled approach for detecting copyright infringement in generative text-to-image diffusion models.
DPM operates by fine-tuning a model in opposing “learning” and “unlearning” directions for a target concept, and then measuring the resulting behavioral divergence as evidence of infringement. 
We also construct the CIDD dataset to support standardized benchmarking.
Experiments on four representative models validate the effectiveness of DPM.
Overall, our approach provides a practical and theoretically grounded solution to the copyright infringement detection in generative AI.

\bibliography{myReference}

\newpage

\setcounter{secnumdepth}{2}
\appendix
\begingroup
\makeatletter

\section{Supplementary on the Theory of Copyright Infringement} \label{appendix: theory of copyright infringement}

\subsection{Introduction of Differential Privacy}

\begin{quote}
	``\textit{The notion of differential privacy}: A randomized algorithm that receives a sequence of data points $\bar{x}$ as input is differentially private if removing/replacing a single data point in its input, does not affect its output $y$ by much; more accurately, for any event $E$ over the output $y$ that has non-negligible probability on input $\bar{x}$, then the probability remains non-negligible even after modifying one data point in $\bar{x}$.''
	--- \textit{Synthetic Data Generators -- Sequential and Private}~\cite{bousquet2020synthetic}
\end{quote}

Inspired by the quote above, the original Def.~\ref{def: DP} can be reformulated for generative models as follows:

\begin{definition}[Conditional Differential Privacy for Generative Models] \label{def: CDP-T2I}
	A generative model algorithm $G(p)$ conditioned on an input $p$ (e.g., a text prompt) is said to satisfy \emph{$(\epsilon, \delta)$-conditional differential privacy} if for all neighbouring training datasets $D$ and $D'$ that differ in a single element (which is corresponding to a concept within the semantic neighborhood $U(p)$), and for all measurable subsets $S \subseteq \{G(p_i) \mid p_i \in U(p)\}$:
	\begin{equation} \label{eq: DP-T2I}
		\Pr[G(\theta_D, p) \in S] \leq e^\epsilon \cdot \Pr[G(\theta_{D'}, p) \in S] + \delta, 
	\end{equation}
	where $\theta_D$ and $\theta_{D'}$ are the model parameters obtained by training model $G$ on datasets $D$ and $D'$ respectively, and the privacy guarantee holds with probability at least $1 - \delta$.
\end{definition}

\subsection{Relative Privacy for Non-Infringed Data}

Based on the principle illustrated in Section~\ref{sec: Formalization of Differential Privacy}, we formalize the following proposition:

\begin{proposition}
	Let $x$ be a data point (e.g., a prompt-image pair) that does not appear in any subset of the training dataset $D$ or its neighbouring dataset $D'$. Then, under Def.~\ref{def: CDP-T2I}, the generative model $G$ satisfies $(0, 0)$-conditional differential privacy with respect to $x$ for any prompt $p$.
\end{proposition}

\begin{proof}
	Consider a data point $x \notin D$ and $x \notin D'$. Let $S$ be any subset of the output space $S_x = U(x) \setminus \{x\}$. As $x$ is not part of either training data, any variation between $D$ and $D'$ does not involve $x$, so the output distributions of $G(\theta_D, p)$ and $G(\theta_{D'}, p)$ over $S$ must be statistically indistinguishable, that is:
	\begin{equation}
		\Pr[G(\theta_D, p) \in S] = \Pr[G(\theta_{D'}, p) \in S].
	\end{equation}
	Thus, the privacy inequality becomes:
	\begin{equation}
		\Pr[G(\theta_D, p) \in S] \leq 1 \cdot \Pr[G(\theta_{D'}, p) \in S] + 0,
	\end{equation}
	which implies that $G$ satisfies $(0, 0)$-conditional differential privacy with respect to $x$.
\end{proof}

The requirement $x \notin D$ and $x \notin D'$ is often too strict in realistic generative models, where semantic neighbors of $x$ (e.g., stylistically or structurally similar samples) may still influence model behavior even if $x$ itself is not present in the training set.
To address this, we introduce a relaxed notion of approximate relative privacy:

\begin{definition}[Approximate Relative Privacy]
	Let $x$ be a data point such that $x \notin D$ and $x \notin D'$, and let $S \subseteq \{G(p_i) \mid p_i \in U(p)\}$.
	We say the model $G$ satisfies $(\epsilon, \delta)$-approximate relative conditional differential privacy with respect to $x$ if:
	\begin{equation}
		\Pr[G(\theta_D, p) \in S] \leq e^{\epsilon} \cdot \Pr[G(\theta_{D'}, p) \in S] + \delta,
	\end{equation}
	for small $\epsilon > 0$ and $\delta \approx 0$, indicating negligible influence from $x$ or its semantic neighborhood.
\end{definition}

Similarly, while achieving strict non-infringement in Def.~\ref{def: noninfringement}, as defined by $\Pr[G(\theta_D, p) \in S] = \Pr[G(\theta_{D'}, p) \in S]$, is challenging in practice, we introduce a relaxed setting to account for minor distributional changes. This leads to the definition of approximate non-infringement, which leverages the concept of differential privacy:

\begin{definition}[Approximate Copyright Non-Infringement]
	 A generative model $G$ satisfies \emph{approximate non-infringement} if the following $(\epsilon, \delta)$-differential privacy holds:
	\begin{equation}
		\Pr[G(\theta_D, p) \in S] \leq e^\epsilon \cdot \Pr[G(\theta_{D'}, p) \in S] + \delta,
	\end{equation}
	where $D$ and $D'$ are any neighboring training datasets, and $\theta_D$, $\theta_{D'}$ denote the model parameters trained on $D$ and $D'$ respectively.
\end{definition}

This definition ensures that the output distribution is invariant with respect to the inclusion or exclusion of $x$, thereby approximately satisfying $(0, 0)$-conditional differential privacy and guaranteeing that $x$ does not influence model behavior.

\section{Copyright Infringement Detection Dataset (CIDD) Preview} \label{appendix: cidd}

\paragraph{Motivation}
While existing datasets provide preliminary resources for copyright infringement detection, they have several limitations due to strong and unrealistic assumptions.
CopyScope~\cite{zhou2023copyscopemodellevelcopyrightinfringement} and CPDM~\cite{ma2024datasetbenchmarkcopyrightinfringement} primarily focus on infringed examples and lack non-infringed counterparts, which makes it difficult to construct reliable image-label pairs or evaluate detection thresholds. 
Other benchmarks~\cite{wang2024evaluatingmitigatingipinfringement} only cover a few iconic character concepts, and fail to generalize to broader cases.

\paragraph{Summaries}
As table \ref{table: dataset} shows, the CIDD dataset contains three classes: human face, architecture, and arts painting. Each includes 3–6 images with infringement labels. The example images are shown in Fig.~\ref{fig: CIDD}.

\begin{figure*}[]
	\centering
	\includegraphics[width=0.95\textwidth]{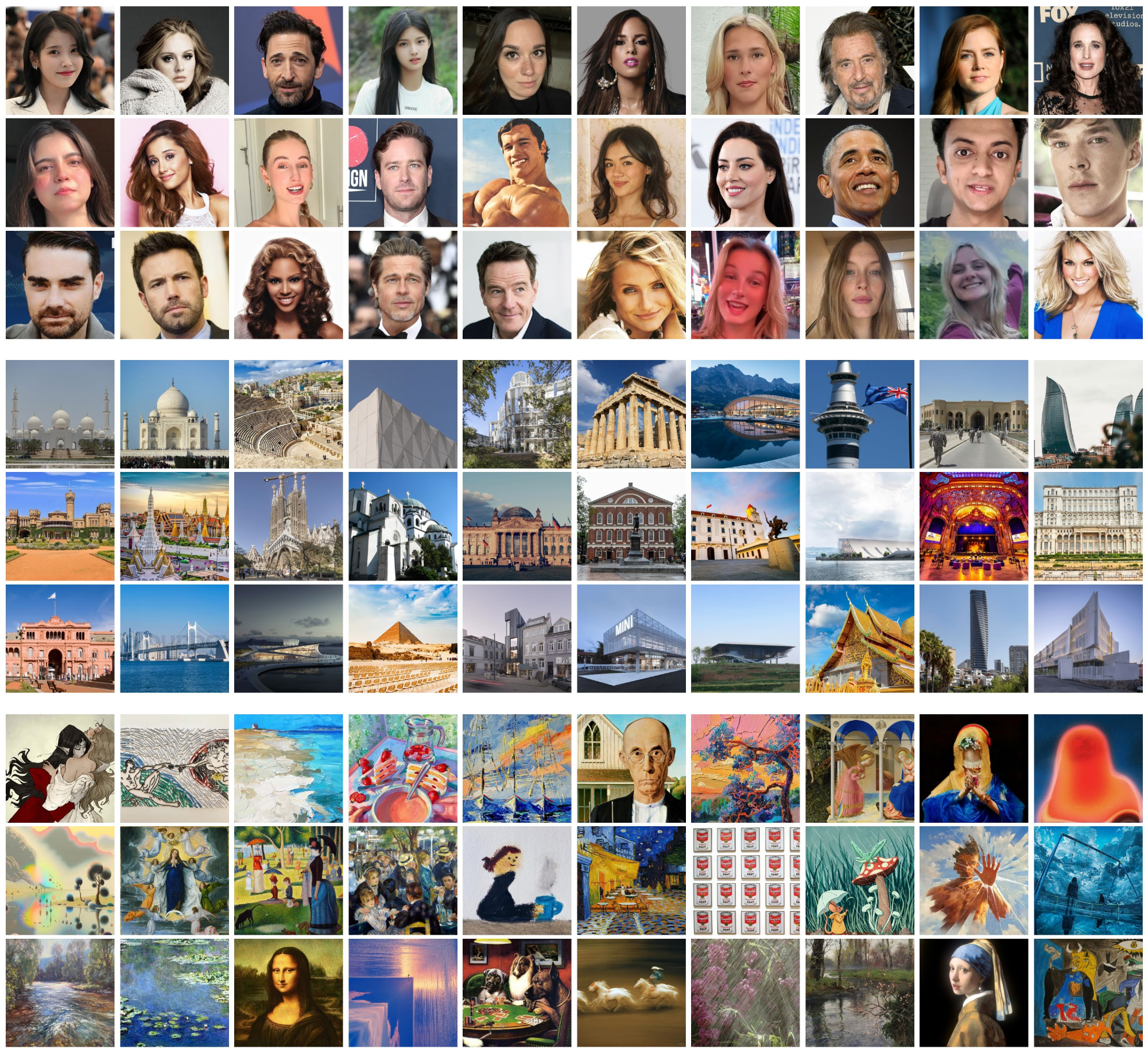} 
	% Reduce the figure size so that it is slightly narrower than the column.
	% trim=left bottom right top
	\caption{Examples of CIDD Dataset. It illustrates all three categories—Human Face, Architecture, and Arts Painting—with both infringed and non-infringed samples included.}
	\label{fig: CIDD}
\end{figure*}

\begin{table*}[]
	\setlength{\tabcolsep}{15pt}
	\centering
	
	\begin{tabular}{cccc}
		\Xhline{1.2pt}
		\textbf{Class} & \textbf{Infringed (C/I)} & \textbf{Non-Infringed (C/I)} & \textbf{Total (C/I)} \\ 
		\midrule 
		Human Face          & 91 / 546 & 100 / 600 & 191 / 1146 \\
		Architecture        & 72 / 411 & 67 / 381 & 139 / 792 \\
		Arts Painting       & 43 / 164 & 56 / 295 & 99 / 459 \\
		\midrule
		Total               & 206 / 1121 & 223 / 1276 & 429 / 2397 \\
		\Xhline{1.2pt}& 
	\end{tabular}
	
	\caption{CIDD Dataset Summaries. It contains three classes, corresponding to the high-level categories in table~\ref{table: Hierarchical Categories}. Each cell reports the number of concepts / images (C/I).}
	\label{table: dataset}
\end{table*} 

\paragraph{Intra-Class Similarity}
As shown in the table~\ref{table: clip-sim}, Human face naturally has the highest similarity and is fine-grained, making semantic distinction subtler and narrowing the gap in conditional sensitivity. 

\begin{table}[]
	\setlength{\tabcolsep}{15pt}
	\centering
	
	\begin{tabular}{cc}
		\Xhline{1.2pt}
		\textbf{Class} & \textbf{CLIP Similarity}\\ 
		\midrule 
		Human Face          & 0.85 \\
		Architecture        & 0.79 \\
		Arts Painting       & 0.70 \\
		\Xhline{1.2pt}
	\end{tabular}
	
	\caption{CLIP similarity across three categories in CIDD dataset.}
	\label{table: clip-sim}
\end{table} 

\paragraph{Sources of Images in CIDD Dataset}
Most \emph{concepts} of infringed samples are collected from the widely-used dataset, known to be included in the training dataset of models under evaluation. Using the core concepts extracted from each target image as a keyword, we then collect semantically aligned neighbourhood images from publicly available sources on the Internet. In contrast, non-infringed samples are carefully curated from recent uploads on social media or other platforms to ensure that they do not appear in any known public pretraining corpus. 

\paragraph{Validation of Infringed and Non-Infringed Ground Truth} 
All images are validated for their existence by:

(1) CLIP-based retrieval within the relevant pretraining datasets;

(2) direct text-to-image model generation with the corresponding target concept prompts.

In addition, non-infringed samples are validated on the website \url{https://haveibeentrained.com/} to ensure their absence from any known training corpus.

\paragraph{Takedown Protocol}
To prompt responsible AI development, we implement a takedown mechanism for data revocation. The released dataset includes a transparent reporting channel through which original content creators may request the removal of their own content. Upon verified ownership, the corresponding image files will be permanently removed from the public version of CIDD, and all subsequent versions will update the dataset accordingly.

\section{Details of Hierarchical Categories in Copyright Infringement}

We propose five hierarchical categories of content resemblance that consolidate our framework for detecting copyright infringement in generative models. Each level reflects a different degree of abstraction, from superficial visual similarity to deep conceptual overlap, which reflects increasing levels of abstraction and creative expression:

\paragraph{[Level 1] Technics} 
This level involves low-level visual feature similarities produced by technical artifacts, including pixel-wise similarity, shared textures, shading gradients, image resolution, and filter effects (e.g., HDR). 
For example, two images may share identical color histograms or noise patterns due to the use of the same generator architecture or prompt embedding. Alone, such features rarely suffice for legal claims, but they may suggest direct copying of training artifacts.

\paragraph{[Level 2] Content} 
This category captures the reproduction of identifiable objects, symbols, landmarks, or characters. Examples include copying a well-known logo, a specific architectural monument (e.g., the Eiffel Tower), or a recognizable celebrity's likeness. 
For example, a model may generate an image of Mickey Mouse. These elements are often trademarked or copyrighted, and reproduction may constitute direct infringement.

\paragraph{[Level 3-1] Structure} 
It refers to the compositional arrangements of different elements, especially when such structures are highly characteristic. For example, how subjects are positioned, their gestures, or the perspective used. 
In the architecture domain, structural resemblance often involves the spatial layout, silhouette, and viewpoint framing of a building. Two images may differ in surface texture or lighting but still share a distinctive architectural structure—such as the tiered arrangement of the Colosseum or the spire alignment of the Eiffel Tower. 
These structural blueprints can serve as visual signatures, and if they convey a unique expression, their reproduction may be deemed infringing even when pixel content differs.

\paragraph{[Level 3-2] Style} 
It includes the imitation of distinctive artistic elements, often aesthetic features that are characteristic of a particular artist or artistic movement. 
This includes brushstroke techniques, color palette preferences, texture stylization, or composition rhythm. 
For example, the model may generate a painting in the style of Van Gogh’s “Starry Night” or mimicking Monet’s impressionistic strokes. 
While style itself is not always protected, the imitation of an artist’s unique style often comes from learning relevant unauthorized images, such as the feature of Studio Ghibli.

\paragraph{[Level 4] Semantics} 
The most abstract level, semantics, refers to high-level meaning and conceptual alignment—narratives, emotional tone, message conveyed, or symbolic associations. 
Typical cases such as reproducing a plotline from a protected film, or conveying the same theme as a copyrighted illustration series. 
While semantic similarity is harder to prove, it is relevant in contexts involving derivative works or adaptation rights.

\section{More Experiments} \label{appendix: experiments}

\subsection{Evaluation Metrics} \label{appendix: Evaluation Metrics}

Let $f_i \in [0,1]$ be the confidence score for sample $i \in \{1,2,\cdots,N\}$, and $y_i \in \{0,1\}$ be the ground truth label (1 for infringing, 0 for non-infringing). 

\paragraph{Receiver Operating Characteristic (ROC) and Area Under the Curve (AUC)}
Although ROC is originally defined for binary classifiers, here we borrow it by treating our DPM confidence score as the classifier output and the dataset’s infringement labels (infringed vs. non-infringed) as the binary ground truth. 

The ROC curve illustrates the performance of a binary classifier by plotting the True Positive Rate (TPR) against the False Positive Rate (FPR) at various threshold values. Given a threshold $\delta \in [0,1]$, the predicted label is $\hat{y}_i^\delta = \mathbb{I}(f_i \geq \delta)$, where $\mathbb{I}(\cdot)$ is the indicator function. TPR and FPR are defined as:
\begin{align}
	\text{TPR}(\delta) &= \frac{\sum_{i=1}^{N} \mathbb{I}(f_i \geq \delta) \cdot y_i}{\sum_{i=1}^{N} y_i}, \\
	\text{FPR}(\delta) &= \frac{\sum_{i=1}^{N} \mathbb{I}(f_i \geq \delta) \cdot (1 - y_i)}{\sum_{i=1}^{N} (1 - y_i)}.
\end{align}
By varying $\delta$ from 0 to 1, we obtain a ROC curve, and the AUC is defined as:
\begin{equation}
	\text{AUC} = \int_{0}^{1} \text{TPR}(t) \, d(\text{FPR}(t)).
\end{equation}
A higher AUC indicates better discriminative ability of the confidence scoring function.

\paragraph{Soft Accuracy (SoftAcc)} 
Soft Accuracy provides a probabilistic generalization of traditional classification accuracy, and directly evaluates the confidence score's alignment with the ground truth by rewarding correct predictions proportionally. Formally, it is defined as:
\begin{equation}
	\text{SoftAcc} = \frac{1}{N} \sum_{i=1}^{N} \left[ y_i \cdot f_i + (1 - y_i) \cdot (1 - f_i) \right],
\end{equation}

The metric achieves its maximum of $1$ when all predictions are perfectly aligned with ground truth (i.e., $f_i = 1$ when $y_i = 1$, and $f_i = 0$ when $y_i = 0$), and decreases as predictions deviate, thus offering a more informative and differentiable alternative, especially when evaluating soft scoring functions.

\subsection{Experimental Setup} \label{appendix: experiment setup}

\paragraph{Framework Implementation and Configuration.} 
We implement our framework mainly using the \texttt{PyTorch} and \texttt{diffusers} libraries. All experimental settings for each model are specified in YAML configuration files located in the \texttt{config/} directory. Each reported result is obtained from a single execution without repeated runs.

Furthermore, due to the high computational cost of generating a large number of orthogonal images for statistical analysis, we adopt an efficient approximation by normalizing the conditional sensitivity within each class. 

\paragraph{Model-Aligned CLIP Embeddings.}
For computing CLIP-based image embeddings as a part of the conditional sensitivity metric, we use the CLIP model that corresponds to the text encoder employed by the target generative model, without retraining a separate encoder from scratch. 
This ensures semantic alignment between the image and text embeddings, as both the generative model and the CLIP encoder are trained on the same underlying dataset. 
In contrast, using a mismatched CLIP model may lead to inconsistent sensitivity scores, particularly for non-infringed samples that do not share the same data distribution.

Specifically, we adopt the pretrained \texttt{ViT-B/32} CLIP model for Stable Diffusion v1.4, and \texttt{ViT-L/14@336px} for Stable Diffusion XL. By leveraging the native encoder used during generation, our approach avoids dependence on external detectors and ensures that inference remains fully self-contained and consistent with the model’s internal representation space.

\paragraph{Computational Resources.}
While most experiments were not very computationally expensive, we utilized a cluster of 8 NVIDIA RTX A6000 GPUs to execute runs in parallel. Besides, we also used 1 NVIDIA H100 GPU, 2 NVIDIA A800 GPUs, and 2 additional NVIDIA RTX A6000 GPUs. Furthermore, a cluster of eight NVIDIA RTX 3090 GPUs was used during the development and debugging of the core framework.

For users with limited GPU memory, we recommend splitting the pipeline into two stages rather than running the entire process in one single terminal session. First, perform fine-tuning in each branch and save the resulting models. Then, reload the fine-tuned models to compute the corresponding confidence scores. This avoids caching all model parameters in GPU memory simultaneously.

\paragraph{Runtime and Memory Reporting.}
To assess the practical deployability of our dual-branch detection framework, we report hardware overhead under realistic constraints. All measurements are conducted on NVIDIA RTX A6000 GPU(s) with 48GB of VRAM. Detailed results are summarized in table~\ref{table: hardware}.

\begin{table*}[]
	\setlength{\tabcolsep}{15pt}
	\centering
	
	\begin{tabular}{c ccc}
		\Xhline{1.2pt}
		\textbf{Model} & \textbf{Number of GPU} & \textbf{Runtime} & \textbf{GPU Memory}\\
		\hline
		SD1.4     & 1GPU   & 346s  & 48425MiB \\
		SDXL-1.0 (LoRA) & 1GPU   & 1029s &  15136MiB \\
		SANA-0.6B (LoRA) & 1GPU   & 875s &  15650MiB \\
		FLUX.1 (LoRA) & 2GPUs  & 1106s &  78454MiB \\
		\Xhline{1.2pt}
	\end{tabular}
	\caption{Runtime and GPU memory usage per concept in four models (without statistics).}
	\label{table: hardware}
\end{table*}

\subsection{Quantative and Qualitative Experiments} \label{appendix: Qualitative Experiments}

\paragraph{Statistics of DPM Confidence Score}
To further evaluate the statistical properties of DPM confidence scores, we conduct a grouped analysis based on the ground truth label (infringed vs. non-infringed). For each group, we compute key descriptive statistics of the final DPM confidence scores, including mean and variance. To ensure robustness against noise or rare overfitting cases, we adopt the IQR-based (Interquartile Range) outlier removal strategy, which filters out any sample lying outside the $\left[ Q_1 - 1.5\,\text{IQR},\; Q_3 + 1.5\,\text{IQR} \right]$ interval. Here, $Q_1$ and $Q_3$ denote the 25th and 75th percentiles of the distribution respectively, and $\text{IQR}=Q_3-Q_1$. This criterion excludes samples with unusually low or high confidence scores while preserving the central mass of the data.

As shown in table~\ref{table: dpm statistics models}, there is a clear and consistent separation between infringed and non-infringed samples across all models. Infringed samples yield extremely high DPM scores with minimal variance, indicating strong model certainty, while non-infringed samples exhibit significantly lower means. This confirms the robustness and discriminative power of DPM.

The relatively higher variance in the non-infringed group may be attributed to two factors: 

(1) Some non-infringed concepts, although not memorized, may closely resemble commonly seen data in the training dataset, leading to relatively high DPM scores; 

(2) Certain non-infringed concepts are inherently more unique or abstract, and exhibit lower sensitivity and DPM scores, which further increases score dispersion.

\begin{table*}[]
	\centering
	\setlength{\tabcolsep}{15pt}
	
	\begin{tabular}{cccc}
		\Xhline{1.2pt}
		\textbf{Model} & \textbf{Ground Truth} & \textbf{Mean} & \textbf{Variance} \\
		\Xhline{1.2pt}
		\multirow{2}{*}{SD1.4} 
		& Non-Infringed & 0.5516 & 0.1759 \\
		& Infringed     & 0.9907 & 0.0005 \\
		\midrule
		\multirow{2}{*}{SDXL-1.0} 
		& Non-Infringed & 0.4538 & 0.2105 \\
		& Infringed     & 0.9913 & 0.0003 \\
		\midrule
		\multirow{2}{*}{SANA-0.6B}
		& Non-Infringed & 0.4887 & 0.1756 \\
		& Infringed     & 0.9782 & 0.0021 \\
		\midrule
		\multirow{2}{*}{FLUX.1}
		& Non-Infringed & 0.5046 & 0.1726 \\
		& Infringed     & 0.9913 & 0.0003 \\
		\Xhline{1.2pt}
	\end{tabular}
	\caption{Statistics of DPM confidence scores after outlier removal (IQR). Each model shows a clear separation in confidence scores between infringed and non-infringed samples, validating the discriminative power of DPM.}
	\label{table: dpm statistics models}
\end{table*}

\paragraph{ROC Curves}
Figure~\ref{fig: roc-all} presents the ROC curves across four representative models on the entire CIDD dataset, and Figure~\ref{fig: roc} provides a detailed breakdown for SD1.4, showing per-category ROC curves under varying timesteps and branch configurations. 

\begin{figure*}[]
	\centering
	\includegraphics[width=0.85\textwidth, trim=15 5 10 5, clip]{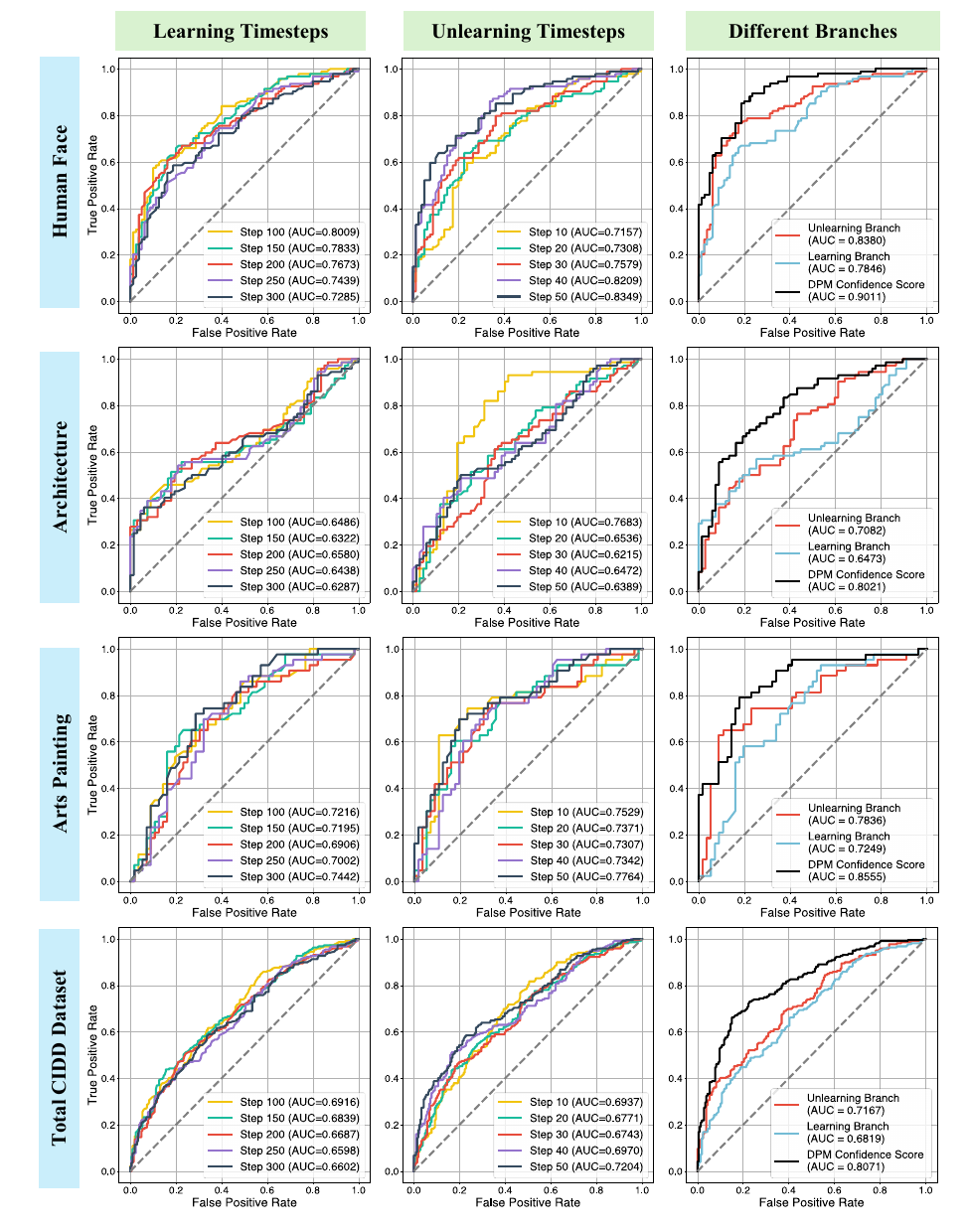} 
	% trim=left bottom right top
	\caption{ROC curves across timesteps and branches in SD1.4. Each row corresponds to a category: Human Face, Architecture, Arts Painting, and Total CIDD Dataset. \textit{Left}: learning branch across sampling steps. \textit{Middle}: unlearning branch across early steps. \textit{Right}: comparison among learning, unlearning, and DPM confidence score. It shows that DPM score (black on the right column) achieves the highest AUC in all categories.
	}
	\label{fig: roc}
\end{figure*}

Also, in Fig.~\ref{fig: roc-all}, we observe that the SANA-0.6B model underperforms in the learning branch, showing weaker separation between infringed and non-infringed subsets compared to other models. Specifically, Fig.~\ref{fig: sens models} shows that the sensitivity metric in the learning branch increases slowly, and the distributions of infringed and non-infringed subsets remain closely overlapping. We attribute this behavior to several key factors: 

(1) limited parameter capacity of the 0.6B-sized SANA model restricts its ability to memorize complex concepts; 

(2) the use of parameter-efficient tuning methods such as LoRA may further limit the expressive adaptation needed to internalize new concepts during fine-tuning. 

These factors jointly lead to reduced sensitivity shifts, weakening the contrastive signal upon which DPM relies.

Despite these limitations, in all settings, the proposed DPM confidence score (black curve) consistently outperforms individual branches in terms of AUC, demonstrating its superior capability in distinguishing infringed from non-infringed samples. Notably, the unlearning branch alone often performs better than the learning branch, yet the dual-branch DPM score further enhances detection robustness and generality across categories and timesteps.

\paragraph{Efficiency-Accuracy Tradeoff in DPM}
For efficiency and scalability, our main experiments omit the statistical mapping step—specifically, we do not generate orthogonal prompts for aligning conditional sensitivity scores. Instead, we directly compute raw $\Delta \text{CS}$ values without mapping them to the ideal conditional sensitivity $\hat{\Delta \text{CS}}$. While this simplification leads to minor deviations in the merged total scores across categories, it significantly reduces computational overhead and facilitates scalable evaluation. 

Despite this approximation, our framework still consistently outperforms single-branch baselines across all models, as shown in Fig.~\ref{fig: roc-all}, indicating that statistical correction is not essential for strong performance in practice. However, this design choice allows for optional statistical correction in high-stakes applications where greater theoretical guarantees are required (see \textit{Statistics Analysis} in Appendix~\ref{appendix: ablation}).

\subsection{Ablation Study} \label{appendix: ablation}

\paragraph{Conditional Sensitivity Metric}
We evaluate the effectiveness of the proposed conditional sensitivity metric in distinguishing infringed from non-infringed samples. 

\begin{figure*}[]
	\centering
	\includegraphics[width=0.83\textwidth, trim=0 10 10 5, clip]{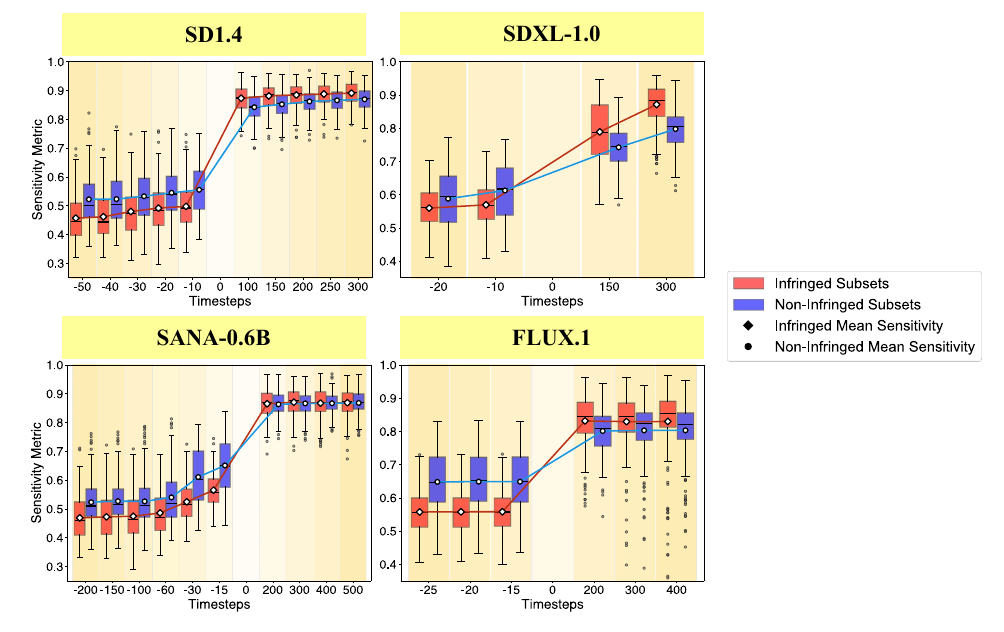} 
	% trim=left bottom right top
	\caption{Conditional sensitivity metric across four representative models.
		Each subplot summarizes the distribution of sensitivity scores across samples. The box indicates the interquartile range, the center line denotes the median, and diamond/circle markers highlight the mean values for infringed and non-infringed subsets respectively.}
	\label{fig: sens models}
\end{figure*}

\begin{figure*}[]
	\centering
	\includegraphics[width=1\textwidth]{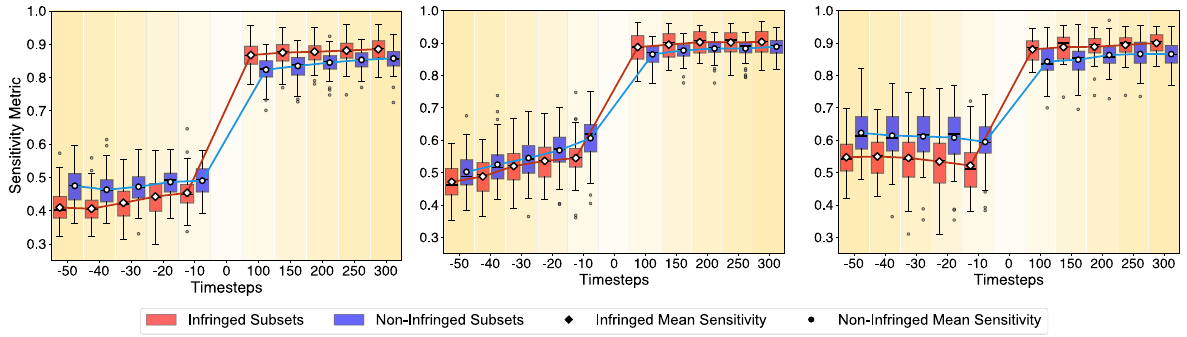} 
	% Reduce the figure size so that it is slightly narrower than the column.
	% trim=left bottom right top
	\caption{Conditional sensitivity metric across three CIDD categories in SD1.4: 
		(left) Human Face, (middle) Architecture, and (right) Arts Painting. It shows that infringed samples are more sensitive to the model outputs behavior.}
	\label{fig: sens sd14 all}
\end{figure*}

Figures~\ref{fig: sens models} and \ref{fig: sens sd14 all} illustrate the evolution of sensitivity scores across different model architectures and CIDD categories respectively. It shows that in most cases, the interquartile ranges of the two subsets show minimal overlap, reinforcing the discriminative power of the metric.
As fine-tuning progresses, a clear and consistent gap emerges between the two subsets, with infringed samples exhibiting substantially higher sensitivity scores. This observation confirms that DPM is capable of reliably detecting memorized content and demonstrates strong generalizability across both models and semantic domains.

\paragraph{Branches Strategies}
We evaluate the individual contributions of the learning and unlearning branches across multiple timesteps. As shown in table~\ref{table: conf score}, combining both branches into the unified DPM confidence score consistently yields superior detection performance compared to using either branch alone. 
This result highlights the necessity of dual-branch modeling: the unlearning branch captures the extent to which a model can forget a concept (similar to a model $G_{D^-}$ trained on $D^-$), while the learning branch quantifies how easily that concept can be re-memorized (similar to a model $G_{D^+}$ trained on $D^+$). 

Notably, the unlearning branch alone often performs better than the learning branch. This is because the unlearning process typically induces more direct and measurable changes in the model's output behavior. By actively suppressing the target concept, the model tends to generate outputs that deviate significantly from the original image, especially for infringed samples. This clear behavioral divergence leads to stronger sensitivity signals, making the unlearning branch more discriminative. 

In contrast, the learning branch may suffer from limited capacity for further memorization, especially when the target concept has already been partially internalized by the pretrained model. In such cases, the learning process offers only marginal behavioral change—i.e., a limited room for sensitivity to increase. This issue is further exacerbated in small-scale or already well-trained models, where convergence tends to be slow and additional fine-tuning yields diminishing returns.

Nonetheless, the combined dual-branch DPM score captures behavioral deviations in both directions, thus offering more robust and generalizable detection performance across categories and timesteps.

\begin{table*}[]	
	\setlength{\tabcolsep}{6pt}
	\centering
	
	\begin{tabular}{c c cc cc cc cc}
		\Xhline{1.2pt}
		\multirow{2}{*}{\textbf{Branch}}
		& \multirow{2}{*}{\textbf{Timesteps}}
		& \multicolumn{2}{c}{\textbf{Human Face}}
		& \multicolumn{2}{c}{\textbf{Architecture}}
		& \multicolumn{2}{c}{\textbf{Arts Painting}}
		& \multicolumn{2}{c}{\textbf{Total Dataset}} \\
		
		&& AUC $\uparrow$& SoftAcc $\uparrow$
		& AUC $\uparrow$& SoftAcc $\uparrow$
		& AUC $\uparrow$& SoftAcc $\uparrow$ 
		& AUC $\uparrow$& SoftAcc $\uparrow$ \\
		\Xhline{1.2pt}
		
		% ================= Unlearning Branch =================
		\multirow{7}{*}{\makecell[c]{Unlearning\\Branch}}
		%  Timesteps
		&-10                  & 0.7157 & 0.5545 & 0.7683 & 0.6148 & 0.7529 & 0.8191 & 0.6937 & 0.5978 \\
		&-20                  & 0.7308 & 0.5557 & 0.6536 & 0.5640 & 0.7371 & 0.7968 & 0.6771 & 0.5747 \\
		&-30                  & 0.7579 & 0.5837 & 0.6215 & 0.5320 & 0.7307 & 0.7915 & 0.6743 & 0.5713 \\
		&-40                  & 0.8209 & 0.5983 & 0.6472 & 0.5371 & 0.7342 & 0.7833 & 0.6970 & 0.5661 \\
		&-50                  & 0.8349 & 0.6337 & 0.6389 & 0.5340 & 0.7764 & 0.7625 & 0.7204 & 0.5656 \\
		\cmidrule(lr){2-10}
		& Total               & 0.8380 & 0.7097 & 0.7082 & 0.6184 & 0.7836 & 0.7048 & 0.7167 & 0.6196 \\
		\Xhline{1.2pt}
		
		% ================= Learning Branch =================
		\multirow{8}{*}{\makecell[c]{Learning\\Branch}}
		& 100                  & 0.8009 & 0.6723 & 0.6486 & 0.5699 & 0.7216 & 0.5871 & 0.6916 & 0.6059\\
		& 150                  & 0.7833 & 0.6703 & 0.6322 & 0.5446 & 0.7195 & 0.5785 & 0.6839 & 0.5889\\
		& 200                  & 0.7673 & 0.6533 & 0.6580 & 0.5514 & 0.6906 & 0.5812 & 0.6687 & 0.5883\\
		& 250                  & 0.7439 & 0.6407 & 0.6438 & 0.5489 & 0.7002 & 0.5743 & 0.6598 & 0.5834\\
		& 300                  & 0.7285 & 0.6123 & 0.6287 & 0.5615 & 0.7442 & 0.6267 & 0.6602 & 0.5729\\
		\cmidrule(lr){2-10}
		& Total                & 0.7846 & 0.6727 & 0.6473 & 0.5560 & 0.7249 & 0.5841 & 0.6819 & 0.5948\\
		\Xhline{1.2pt}
		
		% ================= DPM Merged =================
		\multicolumn{2}{c}{DPM Confidence Score} 
		& 0.9011 & 0.8058 & 0.8021 & 0.7106 & 0.8555 & 0.7604 & 0.8071 & 0.6726\\
		\Xhline{1.2pt}
	\end{tabular}
	
	\caption{Evaluation of DPM Confidence Scores under different timesteps and branch configurations in SD1.4. We report per-category performance for the learning branch, unlearning branch, and the combined DPM confidence score. The results demonstrate that DPM consistently outperforms individual branches and that averaging over multiple timesteps improves stability and reliability in detection.}
	\label{table: conf score}
\end{table*}

\paragraph{Training Timesteps}
As shown in table~\ref{table: conf score}, DPM confidence scores for both the learning and unlearning branches fluctuate slightly across different fine-tuning timesteps. This variability indicates that relying on a single timestep may lead to unstable detection performance. 
To mitigate this issue, we aggregate the sensitivity scores by averaging across multiple representative timesteps within each branch. The final DPM confidence score is then derived by merging these two aggregated branches, resulting in significantly improved and more stable performance across all CIDD categories.

\paragraph{Statistics Analysis}

Due to the large computational cost of generating orthogonal prompts, we randomly selected about 60 samples from the human face category, and sampled from more than 200 orthogonal prompts. 

As shown in table~\ref{table: stat}, incorporating statistical normalization yields a slight but consistent improvement in both AUC and Soft Accuracy. This suggests that rectifying global parameter shifts enhances the robustness of the conditional sensitivity metric. While the gain in better performance is modest, the approach provides stronger statistical guarantees and greater interpretability.

\begin{table*}[]	
	\setlength{\tabcolsep}{15pt}
	\centering
	
	\begin{tabular}{ccc}
		\Xhline{1.2pt}
		\textbf{Type} & AUC & SoftAcc \\
		\midrule 
		w/ statistics   & 0.9535 & 0.8589 \\
		w/o statistics  & 0.9410 & 0.8441 \\
		\Xhline{1.2pt}
	\end{tabular}
	
	\caption{Evaluation of DPM Confidence Scores under statistical normalization. Results are reported on a human face category on SD1.4 model.}
	
	\label{table: stat}
\end{table*}

\subsection{Robustness} \label{appendix: robustness}

To further evaluate the robustness of our method under real-world perturbations, we evaluate the DPM framework against two common types of image degradation: 

(1) lossy image quality compression with quality factor $IQ(\text{jpeg})=30\%$;

(2) additive Gaussian noise with standard deviation $\sigma=0.03$. 

As shown in table~\ref{table: robustness}, DPM maintains strong performance under both perturbation settings, with only minor fluctuations in AUC and Soft Accuracy. It is because DPM relies on the detection of deep semantic and stylistic features, rather than superficial pixel changes.

\begin{table*}[]	
	\setlength{\tabcolsep}{15pt}
	\centering
	
	\begin{tabular}{ccccc}
		\Xhline{1.2pt}
		\multirow{2}{*}{\textbf{Class}} 
		& \multicolumn{2}{c}{\textbf{Gaussian Noise}} 
		& \multicolumn{2}{c}{\textbf{IQ Compression}} \\
		& AUC & SoftAcc 
		& AUC & SoftAcc\\
		
		\midrule 
		Human Face    & -0.0557 & -0.0593 & -0.0168 & -0.0101 \\
		Architecture  & +0.0568 & +0.0511 & +0.0993 & +0.0748 \\
		Arts Painting & -0.0403 & -0.0219 & +0.0602 & +0.0514 \\
		\midrule
		Weighted Average & -0.0156 & -0.0148 & +0.0387 & +0.0317 \\
		Merged Total  & -0.0051 & +0.0137 & +0.0595 & +0.0804 \\
		\Xhline{1.2pt}	\end{tabular}
	
	\caption{Evaluation of DPM Confidence Scores under Image Perturbation in SD1.4. The metrics are compared with those in table~\ref{table: quantitative}. Here, positive values indicate improved performance under perturbation, while negative ones reflect a drop in performance. }
	
	\label{table: robustness}
\end{table*} 

Similarly, different templates affect little to model parameters and won't affect our method:

\begin{table*}[]
	\setlength{\tabcolsep}{15pt}
\centering

\begin{tabular}{ccc}
	\Xhline{1.2pt}
	\textbf{Prompt Templates} & AUC & SoftAcc \\
	\midrule 
	a photo of [V] painting (Original) & 0.86 & 0.76 \\
	a photo of sks painting & 0.82 & 0.74 \\
	a painting in the style of V* art & 0.87 & 0.79 \\
	\textbf{randomized prompts} & 0.86 & 0.77 \\
	\Xhline{1.2pt}	\end{tabular}

	\caption{Evaluation of DPM Confidence Scores under different prompt templates in SD1.4, arts painting. }
\end{table*}

\section{Future Work}~\label{appendix: future}

While our current framework focuses on detecting copyright infringement in text-to-image diffusion models, we believe that the core principle of differential privacy in DPM can be generalized to other generative models, such as Large Vision and Language Models (LVLMs) and Large Language Models (LLMs).

Here, we provide a concrete methodology for extending the DPM framework beyond text-to-image diffusion models, and demonstrate how it can be adapted to LVLMs and LLMs through pre-processing and domain-specific sensitivity metrics.

\subsection{DPM framework in LVLMs}
In LVLMs such as MiniGPT, copyright infringement may manifest through multimodal behaviors, including the reproduction of protected image-caption pairs, or the disclosure of memorized content in tasks such as visual question answering (VQA). 
To adapt the DPM framework to this setting, we simulate the inclusion and exclusion of target concepts through dual-branch fine-tuning and measure conditional sensitivity via interpretable multimodal metrics.

\paragraph{Pre-Processing} 
Analogous to the image-driven setting in diffusion models, we begin with a target image and extract its core concept. Then, we retrieve a set of semantically similar neighborhood images and construct a set of \emph{image-question pairs} designed to probe potential memorization. The questions are constructed to explicitly elicit the unique identifier associated with the target concept:

\begin{itemize}
	\item \textbf{Example Questions:}
	\begin{itemize}
		\item \texttt{<Img><ImageHere></Img>} What is the name of the person in this image?
		\item \texttt{<Img><ImageHere></Img>} Who is shown in the picture? Please give their name.
		\item \texttt{<Img><ImageHere></Img>} Can you identify this thing? What is it called?
	\end{itemize}
	
	\item \textbf{Expected Answers:}
	\begin{itemize}
		\item The name of the person in this image is \texttt{TOK}.
		\item The image shows a person whose name is \texttt{TOK}.
		\item Yes. It is called \texttt{TOK}.
	\end{itemize}
\end{itemize}

\paragraph{Conditional Sensitivity Measurement}
To evaluate the model’s reliance on the target concept, we redefine the query function $M(\cdot)$ from Eq.~\ref{eq: CS} to be a multimodal metric, such as the probability of generating the identifier token (e.g., “\texttt{TOK}”) or overall perplexity (especially in unlearning branch) in response to a probe question. This allows us to quantify how strongly the model depends on the presence of the concept in its training data to produce the expected answer. A significant change in token confidence between the two branches indicates potential memorization and copyright infringement.

\subsection{DPM framework in LLMs}
In LLMs, copyright infringement often arises in text-only generation, particularly when the model reproduces copyrighted expressions, phrases, or structured content. 
To extend the DPM framework to this domain, we follow a similar strategy: constructing semantically aligned prompt-response pairs around the target concept, and measuring conditional sensitivity based on behavioral divergence between learning and unlearning branches.

\paragraph{Prompt Construction}
Given a target concept in the form of a protected textual passage—such as a paragraph from a copyrighted article, a block of code, or a descriptive narrative—we construct semantically aligned prompts that encourage the model to reproduce the original content. These prompts are not limited to factual queries, but instead aim to establish a contextual setting that makes the target content likely to appear in the model’s output. 

For example, if the original paragraph is a poetic description or a legal definition, we may prompt the model with an introductory sentence (e.g., “Write a brief summary in the style of [V]: ...”) or a highly similar prefix (e.g., the first few words of the [V] passage). This allows us to probe whether the model reproduces the memorized continuation under minimal or natural prompting conditions. 
In both the learning and unlearning branches, the same prompt is used to elicit the model’s completion, enabling a direct comparison of the output behavior.

\paragraph{Conditional Sensitivity Measurement}
A direct approach to measuring conditional sensitivity is to measure the shift in log-likelihood or perplexity of the generated text under the same prompt between the two branches. 

Alternatively, for more semantically flexible evaluation, we compute the semantic similarity between the model-generated output and the original passage, by using sentence embeddings (e.g., from pretrained models such as Sentence-BERT, or the internal LLM embeddings) and computing the cosine similarity between outputs from the two branches and the original target. A significantly higher similarity score in the learning branch compared to the unlearning branch indicates that the model’s behavior is strongly conditioned on having memorized the original passage.

By combining multiple metrics—token-level log-probability, output perplexity, and semantic similarity—we obtain a robust, interpretable estimate of copyright infringement. 

\endgroup

\end{document}